\documentclass{article}

 \usepackage[main, final]{neurips_2026}

\usepackage[utf8]{inputenc} 
\usepackage[T1]{fontenc}    
\usepackage{hyperref}       
\usepackage{url}            
\usepackage{booktabs}       
\usepackage{amsfonts}       
\usepackage{amsmath}        
\usepackage{amsthm}         
\usepackage{algorithm}
\usepackage{mathtools}      
\usepackage{algorithmic}    
\usepackage{nicefrac}       
\usepackage{microtype}      
\usepackage{xcolor}         
\usepackage[capitalise,noabbrev]{cleveref} 
\usepackage{multirow}
\usepackage{titletoc}
\usepackage{setspace}

\newcommand{\R}{\mathbb{R}}
\newcommand{\E}{\mathbb{E}}
\newtheorem{definition}{Definition}
\newtheorem{theorem}{Theorem}

\newtheorem{corollary}[theorem]{Corollary}

\newcommand{\toreview}[1]{{\color{black}#1}}
\newcommand{\camerareview}[1]{{\color{black}#1}}

\title{Recursive Semantic Divergence for LLM Agent Consistency}

\author{
Harshavardhan Abichandani \\
SAP SE \\
\texttt{harshavardhan.abichandani@sap.com}
\And
Penny Chong \\
SAP SE \\
\texttt{penny.chong@sap.com}
\And
Atin Ghosh \\
SAP SE \\
\texttt{atin.ghosh@sap.com}
\And
Daniel Dahlmeier \\
SAP SE \\
\texttt{d.dahlmeier@sap.com}
}

\begin{document}

\maketitle

\begin{abstract}
 \camerareview{Large Language Model (LLM) agents with tool use can behave inconsistently across independent runs from identical states. They may call different tools, produce contradictory outputs, or follow divergent strategies.}
 This variability compounds over time and undermines deployment reliability. Existing evaluation metrics primarily measure task correctness via completion-based scores and do not capture consistency in agent behavior across executions. Separately, agent behavior is analyzed either at the token level using probability distributions or at the single-response level using semantic clustering, but neither captures how inconsistency propagates across future turns of a trajectory. We introduce Recursive Semantic Divergence (RSD), a trajectory-level consistency metric adapted from bisimulation metrics in reinforcement learning. RSD measures the expected cumulative semantic disagreement between two independent rollouts of the same agent from the same state, using natural language inference to detect logical inconsistency. It is defined recursively over future states and approximated by a neural network trained on independent rollouts. The learned metric estimates future divergence from the current state, providing a consistency score at each turn without requiring the full trajectory to complete. \camerareview{When used as a reward signal for fine-tuning, RSD reduces AUC coefficient of variation (CV) by up to $30.9\%$ on $\tau^2$-bench and $50.5\%$ on ToolSandbox. Across Qwen-2.5 and Qwen-3.5 models, RSD reduces AUC CV in seven out of eight model-benchmark settings, while improving the final task progress in all eight settings.}\footnote{Source code is publicly available in the repository \href{https://github.com/SAP-samples/agent-consistency-rsd}{https://github.com/SAP-samples/agent-consistency-rsd}}
\end{abstract}

\section{Introduction}

\toreview{Recent advances have enabled large language model (LLM) agents to execute complex workflows and multi-step tasks via tool calling functions \citep{Yao2023ReAct,jang2025dice,liu2023agentbench}. 
These agentic systems are increasingly deployed as autonomous agents that interact with users and tools over multiple turns. 
For long-horizon tasks such as code generation agents \citep{shinn2023reflexion,yang2024sweagent}, as well as agents built for critical applications such as healthcare systems \citep{tu2025towards,kim2024mdagents}, both correctness and reliability are important but manifest in different ways. In long-
horizon tasks, like code generation, small errors in the beginning can accumulate over time, leading to compounding errors that deviate significantly from the expected solution trajectory. 
In critical applications such as healthcare systems, reliability is further challenged by the need for consistent agent behavior in repeated executions. However, in reality, we observe that given the same task to solve, the agent may still yield different decision paths across executions, such as calling entirely different tools, providing
contradictory information to users, or following divergent strategies that lead to different outcomes.

While LLM decoding strategies such as setting the temperature to zero, or using greedy decoding reduce stochasticity at the token level, they do not necessarily ensure consistency at the agent-level. 
For a multi-step or multi-turn LLM agent, the behavior is determined by sequences of decisions conditioned on evolving trajectories, where small differences in the trajectories due to stochastic tool output responses or environmental feedback can lead to divergent agent trajectories even under deterministic decoding strategies. Additionally, even when the underlying LLM is the same, the decoding strategy is deterministic, and the agent interaction history is fixed,  there are additional sources of non-determinism such as floating-point precision and hardware-level variability which can lead to inconsistent agent behavior.

This highlights that the agent behaviour, which is trajectory-dependent,  may vary across executions under identical initial conditions.  Existing evaluation metrics \citep{chang2024agentboard,yao2024tau,lu2024toolsandbox} primarily report task correctness using full or partial task completion metrics, but often do not capture consistency in agent performance. The $pass$\textasciicircum{}$k$ metric \citep{yao2024tau,barres2025tau}, which measures the chance that all $k$ i.i.d. task trials are successful, is the closest existing metric to capturing agent consistency, but it remains binary and is used for reporting rather than 
enforcing consistency in the agent behavior. Another line of work considers either token-level distributions \citep{malinin2020uncertainty} or single-response semantic similarity \citep{kuhn2023semantic}, but these approaches do not capture how inconsistencies propagate across future turns in a trajectory. Notably, agent consistency is independent of correctness, as agents may exhibit stable behavior patterns that are either consistently successful or consistently failing across executions.

Motivated by the need for a reliable and consistent agent behavior in long-horizon tasks and critical applications, and the  trajectory-dependent variability of LLM agents that is not captured by existing evaluation metrics, we propose a trajectory-level consistency metric to enforce agent consistency conditioned on identical interaction-history states.
We introduce Recursive Semantic Divergence (RSD), an unsupervised trajectory-level consistency metric adapted from bisimulation metrics \citep{Ferns2004Metrics} in reinforcement learning, which define state similarity in terms of behavioral equivalence. We extend this framework to interaction-history states and replace reward differences with a semantic discrepancy based on Natural Language Inference (NLI) \citep{Bowman2015Large}, adapted to LLM agents' actions. 
We then define a recursive bisimulation-style objective over this semantic distance, yielding a fixed-point formulation of behavioral consistency, which is approximated using the Matching under Independent Coupling (MICo) method \citep{Castro2021MICo}. We show that training with RSD metric as a reward signal improves agent consistency. 

\begin{figure}[t]
    \centering
    \includegraphics[width=\textwidth]{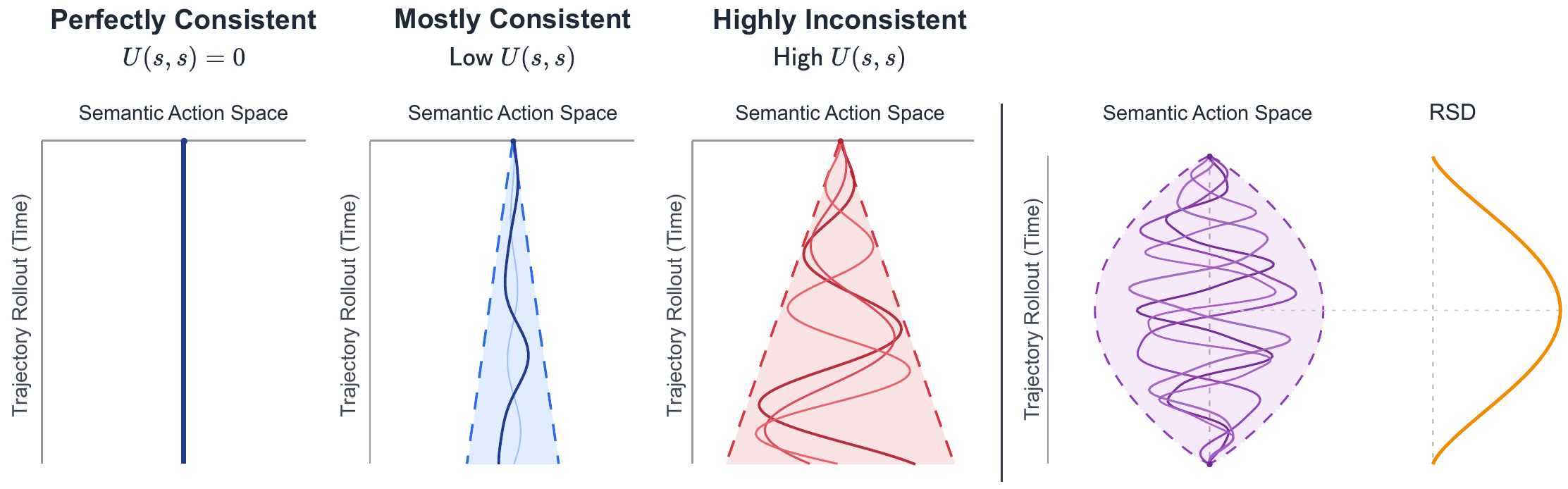}
    \caption{Diagram illustrating Recursive Semantic Divergence (RSD) in LLM agents. The first three panels display trajectory spreads for perfectly consistent, mostly consistent, and highly inconsistent agents over time. The final panel maps an agent's expanding and converging semantic action space alongside its corresponding RSD score, demonstrating how the metric tracks the variance of semantic divergence at each turn. }
    \label{fig:concept_diagram}
\end{figure}

We summarize our contribution as follows:
\begin{enumerate}
\item We introduce Recursive Semantic Divergence (RSD), an unsupervised trajectory-level consistency metric that adapts bisimulation from reinforcement learning to evaluate LLM agents.
\item We empirically demonstrate that using the learned RSD metric as a reward signal for fine-tuning significantly reduces behavioral variation and improves consistency across benchmarks and model scales.
\item We provide a theoretical analysis showing that the RSD fixed-point formulation converges and yields meaningful guarantees about agent behavior. In particular, we derive bounds on the probability of total trajectory disagreement and prove the existence of a dominant action class when the agent exhibits low RSD.
\end{enumerate}


} 

\section{Related Work}

\toreview{\textbf{Agent evaluation.} Existing agent evaluation methods primarily focus on task correctness or completion rates. \citet{wang2023mint,xie2024osworld} evaluate coarse-grained agent performance using success-based metrics, while \citet{chang2024agentboard,chong2026talk} report fine-grained metrics that capture intermediate progress and turn-level efficiency. Representing the tool call dependencies of an agent with a directed acyclic graph (DAG) \citep{lu2024toolsandbox}, is also used to evaluate agent performance but it does not assess agent consistency across multiple executions. Although some prior works partially account for agent stochasticity when reporting final performance metrics, they do not explicitly measure consistency in agent behavior or propose methods to enforce consistent execution of agent trajectories. For instance, \citet{yao2024tau,barres2025tau}, measure agent reliability with the $pass$\textasciicircum{}$k$ metric which quantifies the chance that all $k$ i.i.d. task trials are successful out of the $n$ total trials, but the metric remains limited for binary task outcomes only and not for evaluating the consistency of intermediate progress. Similarly, \citet{chong2026talk}, analyse the agent consistency across the different executions by visualising expected agent progress and variance in a plot, but  does not provide method to improve agent consistency, which is the primary goal of our work.
}

\toreview{
\textbf{Uncertainty estimation.} Quantifying uncertainty in LLM outputs has been approached from multiple angles. \citet{malinin2020uncertainty} measure token-level entropy, but such methods require access to model logits or weights and fail to account for the semantic content of the output. \citet{kuhn2023semantic} introduced semantic entropy, in which sampled responses are semantically clustered before entropy is computed. While this addresses the limitations of token-level measures, it operates at the level of single-query responses and does not capture how disagreement at one turn propagates across a multi-turn trajectory.

For multi-step agents specifically, \citet{Kirchhof2025Position} argue that uncertainty quantification methods require fundamental reassessment due to underspecification and environmental effects introduced during agent interaction. \citet{Duan2025UProp} propose an information-theoretic framework that decomposes uncertainty propagation across agent steps into intrinsic and extrinsic components. However, their approach is primarily diagnostic, requires the agent to complete the full trajectory before analysis, and does not produce a single scalar score that can be directly used as a reward signal. In contrast, RSD is a lookahead, unsupervised metric that captures how semantic inconsistency propagates across future agent turns.
}

\section{Preliminaries}

\subsection{Bisimulation Metrics}
\label{sec:bisim}

Bisimulation metrics \citep{Larsen1991BisimulationTP, Ferns2004Metrics} were introduced in reinforcement learning to measure the behavioral distance between MDP states. Two states are bisimilar if and only if they receive the same immediate reward and they transition to states that are themselves bisimilar. This recursive definition yields a fixed-point equation. Following \citet{Ferns2004Metrics}, for an MDP with reward $R$, transition kernel $P$, and discount factor $\gamma$, the bisimulation distance $d^*$ is the unique fixed point of the operator $\mathcal{T}$ acting on the space of all state-pair distance functions $\mathcal{D}$. We define $\mathcal{T}: \mathcal{D} \to \mathcal{D}$, $d \mapsto \mathcal{T}d$, with pointwise form
\toreview{
\begin{equation}
    (\mathcal{T}d)(s^{(i)}, s^{(j)}) = |R(s^{(i)}) - R(s^{(j)})| + \gamma \cdot W(d)\big(P(\cdot \mid s^{(i)}, \pi),\; P(\cdot \mid s^{(j)}, \pi)\big),
\label{eq:bisim}
\end{equation}
}
and the bisimulation distance is given by the fixed point $d^* = \mathcal{T}d^*$, where $W(d)(\cdot, \cdot)$ denotes the Kantorovich distance between the two successor-state distributions \toreview{$P(\cdot \mid s^{(i)}, \pi)$} and \toreview{$P(\cdot \mid s^{(j)}, \pi)$}, under the common policy $\pi$, using the ground metric $d$. States that are close under $d^*$ produce similar reward sequences and similar future trajectories, regardless of how different they appear in observation space.

\subsection{MICo: Sampling-Based Bisimulation}
\label{sec:mico}

Computing the exact fixed point of \Cref{eq:bisim} requires evaluating the \toreview{Kantorovich distance over successor state distributions, which is intractable when these distributions are accessible only via samples.} The Matching under Independent Coupling (MICo) method \citep{Castro2021MICo} addresses this by replacing the intractable computation with sample-based coupling: rather than computing 
\toreview{$W(d)(P(\cdot|s^{(i)},\pi ), P(\cdot|s^{(j)}, \pi))$ }over the full distribution, it pairs individual successor samples \toreview{$(s'^{(i)}, s'^{(j)})$ } drawn independently from their respective transition distributions. The key insight from \citet{Castro2021MICo} is that training on independently coupled samples still converges to a distance function that reflects behavioral similarity.

MICo parameterizes the distance through a learned embedding $\phi_\theta: \R^d \to \R^m$ that maps state representations to a latent space. The distance between two states combines a norm component (capturing individual state ``spread'') and an angular component (capturing relative orientation):
\toreview{
\begin{equation}
    \hat{U}_\theta(s^{(i)}, s^{(j)}) = \alpha\big(\|\phi_\theta(z^{s^{(i)}})\|^2 + \|\phi_\theta(z^{s^{(j)}})\|^2\big) + \beta \cdot d_{\angle}\big(\phi_\theta(z^{s^{(i)}}),\; \phi_\theta(z^{s^{(j)}})\big),
\label{eq:mico_dist}
\end{equation}
}
where $z^s  \toreview{\in \R^d}$ is a pretrained text embedding of state $s$, $d_{\angle}$ is the angular distance (normalized arccos of cosine similarity), and $\alpha, \beta > 0$ are fixed coefficients.



\section{Methodology}

\subsection{Problem Formulation}
\label{sec:pomdp}

We follow existing methods \citep{Yao2023ReAct, yao2024tau, huang2025crmarena} of formulating the LLM agent environment as a Partially Observable Markov Decision Process (POMDP) \citep{Kaelbling1998Planning}:
\begin{equation}
  \mathcal{M} = \langle \mathcal{S}, \mathcal{A}, \mathcal{O}, T, R \rangle, 
\end{equation}
\toreview{where
$\mathcal{S}$ is the set of true environment states such as database contents and user preferences, 
$\mathcal{A}$ is the agent's action space which includes the agent responses and tool calls, and
$\mathcal{O}$ is the observation space which includes all feedback the agent receives from the environment such as user utterance and tool output. The transition function $T: \mathcal{S} \times \mathcal{A} \to \mathcal{S}$ governs how the environment changes after the agent takes an action. The reward function $R: \mathcal{S} \times \mathcal{A} \to \mathbb{R}$ defines the reward the agent receives for executing a particular action in that state. This reward may be difficult to obtain without ground truth or human feedback.

}
\toreview{In a standard setting, LLM agents never observe  state $s \in \mathcal{S}$ directly, due to partial observability, context limitations, and potential access restrictions.
Instead, the agent conditions on the interaction history, which includes the complete sequence of observations received \toreview{$o_0, o_1, ... \in \mathcal{O}$} and actions taken \toreview{$a_0, a_1, ... \in \mathcal{A}$}:
}
\begin{equation}
  h_t = (o_0, a_0, o_1, a_1, \ldots, o_{t-1}, a_{t-1}, o_{t}).
\end{equation}

We reduce this POMDP to an MDP over belief states. For LLM agents the model's policy is already a function of the conversation history. Therefore, the history $h_t$ serves as the belief state, and we define the induced MDP $\mathcal{M}' = \langle \mathcal{H}, \mathcal{A}, \toreview{T'} \rangle$ where $\mathcal{H}$ is the set of all reachable histories and the transitions \toreview{$T'$} are deterministic concatenation of the observations: $h_{t+1} = h_t \oplus a_t \oplus o_{t+1}$. We denote a state in this induced MDP as \toreview{$s_t \equiv h_t$}.

The underlying LLM induces a stochastic policy $\pi(a \mid s)$ over this MDP, which results in a distribution over future trajectories $\tau \sim \Pi(\cdot \mid s)$. This means that two independent samples from this distribution may produce semantically different sequences. We refer to this semantic spread of sampled trajectories from the same state as inconsistency. Lower spread indicates that the policy is more deterministic.

\subsection{Recursive Semantic Divergence}
\label{sec:rsd}

The bisimulation framework (\Cref{sec:bisim}) supplies a recursive, discounted measure of behavioral similarity that captures how present differences propagate into the future. This notion maps directly onto agent consistency: if two independent runs from the same state yield similar behavioral futures, the agent is consistent; if they yield divergent futures, it is not.

\toreview{In the absence of human feedback or ground truth labels for LLM agents, we replace the immediate reward difference in a standard MDP, $|R{(s^{(i)})} - R(s^{(j)})|$ from \Cref{eq:bisim} with a Natural Language Inference (NLI)-based distance function. We adopt 
NLI \citep{Bowman2015Large, Williams2018Broad}, which classifies sentence pairs as entailment, contradiction, or neutral and use it as an unsupervised measure of semantic agreement between agent actions, where bidirectional entailment is interpreted as equivalence, contradiction as disagreement, and neutrality as partial overlap.
We define the following NLI distance function:}

\begin{definition}[NLI Distance Function]
\label{def:nli}
For agent actions $\toreview{a^{(i)}, a^{(j)}} \in \mathcal{A}$:
\begin{equation}
    d_{\text{NLI}}(\toreview{a^{(i)}, a^{(j)}}) = \begin{cases}
    0 & \text{if bidirectional entailment}. \\
    1 & \text{if contradiction in either direction}. \\
    0.5 & \text{otherwise (neutral).}
    \end{cases}
\end{equation}
\end{definition}

This function is symmetric, identity-preserving ($d_{\text{NLI}}(a,a) = 0$), and bounded in $[0,1]$.\camerareview{For tool-calling actions, entailment indicates functionally equivalent actions, such as calling the same tool with equivalent arguments, while contradiction indicates incompatible actions, such as calling the same tool with different arguments that lead to different outcomes. Neutral captures actions whose functional equivalence is ambiguous.} \camerareview{Considering two independently sampled trajectories, where $t$ denotes the transition depth along each trajectory, and substituting $d_{\text{NLI}}$ for the immediate reward difference in \Cref{eq:bisim}, we obtain the following recursive distance:}

\begin{equation}
    U(\toreview{s_t^{(i)}, s_t^{(j)}}) = \E_{\toreview{a_t^{(i)}} \sim \pi(\cdot|s_t^{(i)}),\; \toreview{a_t^{(j)}} \sim \pi(\cdot|s_t^{(j)})}\big[d_{\text{NLI}}(\toreview{a_t^{(i)}, a_t^{(j)}})\big] + \gamma \cdot \E\big[U(\toreview{s_{t+1}^{(i)}, s_{t+1}^{(j)}})\big],
\label{eq:rsd_general}
\end{equation}
\toreview{where $s_{t+1}^{(i)} \sim P(\cdot | s_t^{(i)}, a_t^{(i)} )$ and $s_{t+1}^{(j)} \sim P(\cdot | s_t^{(j)}, a_t^{(j)} )$ are successor states.}


\subsection{Learned Approximation}
\label{sec:approximation}

Computing the exact fixed point of \Cref{eq:rsd_general} is intractable for two reasons. First, we treat the LLM as a black box and do not have access to the internal transition probabilities \toreview{$P(s_{t+1} \mid s_t, a_t)$}. Second, even with access, the state space (all possible \toreview{interaction} histories) is computationally large.

\toreview{To address this, we approximate the fixed point of \Cref{eq:rsd_general} using the MICo method (\Cref{sec:mico}) with temporal-difference (TD) learning using the rollouts from the LLM agent. 
Given paired trajectory rollouts $\tau^{(i)}$ and $\tau^{(j)}$ from the same task, \camerareview{we pair the state-action transitions $(s_t^{(i)}, a_t^{(i)})$ and $(s_t^{(j)}, a_t^{(j)})$ at the same rollout depth $t$. The shared index denotes the transition depth along each rollout. After divergence, each trajectory independently follows its own successor states.}
The $n$-step TD target $y_t$ is constructed by unrolling both trajectories for the next $n$ steps:
\begin{equation}
    y_t= \sum_{k=0}^{n-1} \gamma^k d_{\text{NLI}}(a_{t+k}^{(i)}, a^{(j)}_{t+k}) + \gamma^n \hat{U}_{\bar{\theta}}(s_{t+n}^{(i)}, s^{(j)}_{t+n}) ,  
\label{eq:td_y}
\end{equation}
where $\bar{\theta}$ is a copy of the network parameters that is soft-updated at each step with $\theta$ and $\gamma$ is the discount factor. \camerareview{If either trajectory terminates within the $n$-step window, the TD target accumulates disagreement only over the valid paired transitions, and the bootstrap term is omitted. No padding or actions beyond termination are introduced.} The target corresponds to an $n$-step unrolling of the Bellman equation, where discrepancies are accumulated and exponentially discounted with the time step.
The network is trained by minimizing the following Bellman residual:
\begin{equation}
    \mathcal{L}(\theta) = \E\left[\left(\hat{U}_\theta(s^{(i)}_t, s^{(j)}_t) - y_t\right)^2\right].
\label{eq:td_loss}
\end{equation}}

\subsection{Self-Distance as a Measure of Consistency}
\label{sec:self_distance}

The MICo distance function from \Cref{eq:mico_dist} produces what \citet{Castro2021MICo} term a \textit{diffuse} distance---the self-distance $U(s, s)$ can be non-zero. We use this property to measure how stochastic the process is from a given state. If the transition dynamics from state $s$ have many possible successor states, then $U(s, s) > 0$. For deterministic states, $U(s, s) = 0$. Intuitively, it captures how wide the tree of possibilities is when unrolling the process from that state.

For the self-distance $\hat{U}_\theta(s, s)$, the angular term vanishes (a vector has zero angle with itself), leaving:
\begin{equation}
    \hat{U}_\theta(s, s) = 2\alpha\|\phi_\theta(\toreview{z^s})\|^2.
\label{eq:self_distance}
\end{equation}

In our formulation, stochasticity arises from the LLM policy $\pi(\cdot \mid s)$, which may sample different actions across runs. \camerareview{Since we replace the reward difference with NLI-based semantic disagreement}, the self-distance $U(s,s)$ captures the expected cumulative \textit{semantic} divergence between two independent trajectories originating from the same conversational state. It is zero when the agent always produces semantically equivalent actions (regardless of surface variation in phrasing), and positive when the agent exhibits semantic disagreement across independent runs.

\begin{definition}[Recursive Semantic Divergence]
\label{def:rsd}
The Recursive Semantic Divergence at state $s$ is defined as \toreview{the infinite-horizon discounted sum of stepwise discrepancies obtained by expanding \Cref{eq:rsd_general}}:
\begin{equation}
    U(s, s) = \E_{\tau, \tau' \sim \Pi(\cdot|s)}\left[\sum_{t=0}^{\infty} \gamma^t \cdot d_{\text{NLI}}(a_t, a'_t)\right], 
\end{equation}
\toreview{where $\tau$ and $\tau'$ are independent trajectory rollouts from the distribution of trajectories $\Pi(\cdot | s )$, approximated in practice by $\hat{U}_\theta(s,s) = 2\alpha\|\phi_\theta(z_s)\|^2$.}
\end{definition}

RSD provides a notion of how consistent an LLM agent is. Consider a state $s$ from which the agent's policy $\pi(\cdot|s)$ places all probability mass on a single semantic equivalence class of actions (e.g., it always calls the same tool with the same arguments, possibly with different phrasings). In this case, $d_{\text{NLI}}(a, a') = 0$ for any two independent samples, and by the Bellman structure, if all reachable successor states also exhibit this concentration, then $U(s,s) = 0$. The agent is perfectly consistent from state $s$ onward.


\toreview{
The complete RSD training algorithm is detailed in \Cref{alg:rsd_training} in Appendix \ref{Algorithm for RSD training and inference}.
For inference, the computation of the consistency score requires only a single forward pass, with no additional rollouts, NLI evaluations, or sampling. The Appendix \ref{Algorithm for RSD training and inference} also details the algorithm for RSD inference in \Cref{alg:rsd_inference}.
}

\section{Theoretical Analysis}
\label{sec:theory}

Our formulation replaces the reward signal with NLI classifications and operates on compressed state embeddings rather than full conversation histories. We need to establish two things: first, that the recursive equation converges to a well-defined quantity, and second, that this quantity provides meaningful guarantees about the agent's actual behavior.

The Bellman operator in \Cref{eq:rsd_general} is a $\gamma$-contraction in supremum norm, so a unique fixed point $U^*$ exists with $U^* \in [0, 1/(1-\gamma)]$. Refer to \Cref{appendix:theory}

\subsection{Disagreement Probability}

In this section we interpret the $U(s, s)$ score and make conclusions about the probability that two independent runs of the agent disagree. We first show a bound on the total trajectory level disagreement and a bound on the existence of a dominant behavioral mode.

\begin{theorem}[Trajectory Agreement]\label{thm:agreement_main}
Let $D(\tau^{(i)}, \tau^{(j)}) = \sum_{t=0}^{\infty} \gamma^t d_{\text{NLI}}(a_t^{(i)}, a_t^{(j)})$ be the total discounted disagreement between two independent rollouts from state $s$. For any threshold $\delta > 0$:
\begin{equation}
    \Pr\big(D(\tau^{(i)}, \tau^{(j)}) \geq \delta\big) \leq \frac{U(s,s)}{\delta}.
\end{equation}
\end{theorem}

\begin{theorem}[Dominant Action Class]\label{thm:dominant_main}
If $U(s,s) < 1/4$, there exists a semantic equivalence class $[C^*]$ such that:
\begin{equation}
    \Pr(\pi(s) \in [C^*]) \geq 1 - 2U(s,s).
\end{equation}
\end{theorem}

\camerareview{Since all disagreement terms in \Cref{def:rsd} are non-negative, the expected immediate disagreement $D_0(s) = \E_{a,a' \sim \pi(\cdot|s)}
[d_{\text{NLI}}(a,a')]$
satisfies $D_0(s) \leq U(s,s)$. Thus, a low trajectory-level RSD also bounds immediate action disagreement. This guarantee assumes that bidirectional entailment induces well-defined semantic equivalence classes. In practice, imperfect NLI classifications may violate this assumption, and we analyze the effect of approximation errors in \Cref{appendix:theory}.}

Together, we can interpret the results as, if $U(s,s) = 0.1$, \camerareview{the probability that two independent actions sampled from state $s$ contradict each other is at most 10\%,} and at least 80\% of the policy's mass is concentrated on a single semantic action class.

\toreview{The proofs for \Cref{thm:agreement_main} and \Cref{thm:dominant_main}, and other theorems are detailed in Appendix \ref{appendix:theory}.}

\begin{table}[H]
\caption{Performance and consistency, measured by the coefficient of variation (CV), across different fine-tuning strategies. The upwards and downwards arrows $\uparrow \downarrow$ show relative change from the baseline. The blue downwards arrow \textcolor{blue}{$\downarrow$} indicates reduction in variation, while the red upwards arrow \textcolor{red}{$\uparrow$} indicates increase in variation. }
\label{tab:main_results}
\centering
\small
\begin{tabular}{llcccccc}
\toprule
& \textbf{Model} & \textbf{AUC} & \textbf{AUC CV} & \textbf{Progress} & \textbf{Progress CV} & \textbf{$pass$@$1$} & \textbf{$pass$\textasciicircum{}$3$} \\
\midrule
\multicolumn{8}{c}{\textit{ToolSandbox Qwen-2.5-3B}} \\
\midrule
& Base Qwen (baseline) & 0.271 & 0.90 & 0.293 & 0.86 & 14.1\% & 20.0\% \\

& Token Entropy & 0.302 & 0.50{\scriptsize\textcolor{blue}{$\downarrow$43.8\%}} & 0.317 & 0.53{\scriptsize\textcolor{blue}{$\downarrow$37.8\%}} & 15.6\% & 20.0\% \\

& Semantic Entropy & 0.326 & 0.83{\scriptsize\textcolor{blue}{$\downarrow$7.0\%}} & 0.354 & 0.82{\scriptsize\textcolor{blue}{$\downarrow$4.0\%}} & 17.2\% & 20.0\% \\

& SFT & 0.311 & 0.82{\scriptsize\textcolor{blue}{$\downarrow$8.9\%}} & 0.327 & 0.81{\scriptsize\textcolor{blue}{$\downarrow$5.6\%}} & 17.2\% & 20.0\% \\

& RSD & \textbf{0.340} & \textbf{0.44}{\scriptsize\textcolor{blue}{$\downarrow$50.5\%}} & 0.362 & \textbf{0.43}{\scriptsize\textcolor{blue}{$\downarrow$49.3\%}} & 18.8\% & 21.4\% \\
\midrule  
\multicolumn{8}{c}{\textit{ToolSandbox Qwen-2.5-7B}} \\
\midrule
& Base Qwen (baseline) & 0.539 & 0.38 & 0.557 & 0.36 & 29.7\% & 28.6\% \\

& Token Entropy & 0.510 & 0.50{\scriptsize\textcolor{red}{$\uparrow$32.4\%}} & 0.523 & 0.50{\scriptsize\textcolor{red}{$\uparrow$40.5\%}} & 21.9\% & 19.4\% \\

& Semantic Entropy & 0.504 & 0.36{\scriptsize\textcolor{blue}{$\downarrow$6.2\%}} & 0.512 & 0.34{\scriptsize\textcolor{blue}{$\downarrow$6.1\%}} & 23.4\% & 23.2\% \\

& SFT & 0.518 & 0.45{\scriptsize\textcolor{red}{$\uparrow$17.9\%}} & 0.530 & 0.43{\scriptsize\textcolor{red}{$\uparrow$20.6\%}} & 31.2\% & 28.8\% \\

& RSD & \textbf{0.555} & \textbf{0.28}{\scriptsize\textcolor{blue}{$\downarrow$26.8\%}} & 0.562 & \textbf{0.27}{\scriptsize\textcolor{blue}{$\downarrow$23.8\%}} & 25.0\% & 28.6\% \\  

\midrule
\multicolumn{8}{c}{\textit{$\tau^2$-bench Qwen-2.5-3B}} \\
\midrule

& Base Qwen (baseline) & 0.200 & 1.32 & 0.240 & 1.25 & 9.4\% & 0.3\% \\

& Token Entropy & 0.204 & 1.38{\scriptsize\textcolor{red}{$\uparrow$3.9\%}} & 0.247 & 1.32{\scriptsize\textcolor{red}{$\uparrow$6.1\%}} & 9.4\% & 0.3\% \\

& Semantic Entropy & 0.180 & 1.11{\scriptsize\textcolor{blue}{$\downarrow$15.9\%}} & 0.239 & 1.00{\scriptsize\textcolor{blue}{$\downarrow$20.0\%}} & 6.2\% & 1.0\%\\

& SFT & \textbf{0.289} & 0.95{\scriptsize\textcolor{blue}{$\downarrow$28.0\%}} & 0.353 & \textbf{0.90}{\scriptsize\textcolor{blue}{$\downarrow$27.8\%}} & 15.6\% & 5.4\% \\

& RSD & 0.274 & \textbf{0.91}{\scriptsize\textcolor{blue}{$\downarrow$30.9\%}} & 0.332 & 0.92{\scriptsize\textcolor{blue}{$\downarrow$26.1\%}} & 15.6\% & 3.6\% \\

\midrule
\multicolumn{8}{c}{\textit{$\tau^2$-bench Qwen-2.5-7B}} \\
\midrule
& Base Qwen (baseline) & 0.249 & 0.78 & 0.309 & 0.74 & 11.2\% & 0.6\% \\

& Token Entropy & 0.304 & 0.75{\scriptsize\textcolor{blue}{$\downarrow$3.7\%}} & 0.387 & \textbf{0.70}{\scriptsize\textcolor{blue}{$\downarrow$4.6\%}} & 15.6\% & 1.2\% \\

& Semantic Entropy & 0.242 & 0.78{\scriptsize\textcolor{red}{$\uparrow$0.5\%}} & 0.296 & 0.74{\scriptsize\textcolor{gray}{$\pm$0\%}} & 7.8\% & 0.0\%  \\

& SFT & \textbf{0.366} & \textbf{0.73}{\scriptsize\textcolor{blue}{$\downarrow$6.5\%}} & 0.462 & 0.72{\scriptsize\textcolor{blue}{$\downarrow$1.9\%}} & 25.0\% & 7.1\% \\

& RSD & 0.278 & 0.89{\scriptsize\textcolor{red}{$\uparrow$14.2\%}} & 0.349 & 0.85{\scriptsize\textcolor{red}{$\uparrow$15.1\%}} & 12.5\% & 3.0\% \\

\bottomrule
\end{tabular}
\end{table}

\section{Experimental Setup}
\label{main_experimental_setup}

\toreview{We evaluate RSD on $\tau^2$-bench \citep{barres2025tau}, and ToolSandbox \citep{lu2024toolsandbox} to measure whether the learned metric captures trajectory-level consistency and if it can serve as a reward signal for improving consistency during fine-tuning.  For each benchmark, we run 8 independent trials across 8 test samples (64 trajectories) in total for Qwen-2.5 models (3B and 7B parameters) \citep{qwen2.5}, where each trajectory consists of interleaved user messages, agent actions (including tool calls), and tool outputs. We construct the training set by generating multiple rollouts using temperature set to 1.0. We compare RSD against the Qwen baseline and three other methods, i.e., token-level entropy \citep{malinin2020uncertainty}, semantic entropy \citep{kuhn2023semantic}, and supervised fine-tuning (SFT). For token entropy, semantic entropy, and RSD experiments, we fine-tune the base models using Group Relative Policy Optimization (GRPO) \citep{Shao2024DeepSeekMath}, using each metric as a reward signal. We minimize token and semantic entropy via negative entropy rewards, and minimize RSD self-distance $-\hat{U}_\theta(s,s)$. \camerareview{We also evaluate RSD on Qwen-3.5 \citep{qwen3.5} across both benchmarks, comparing RSD fine-tuning against their corresponding base models. Refer to \Cref{appendix:qwen35} for results.}

For the evaluation metric, we adopt the evaluation framework from \citet{chong2026talk} to evaluate the agent behaviour. We report two task performance metrics, area under the progress curve (AUC) which measures how quickly the agent makes progress, and the final progress rate. 
To measure behavioral consistency, we compute the coefficient of variation (CV = $\sigma / \mu$) of performance metrics across 8 independent trials per task, where a lower CV indicates more consistent behavior. To ensure a fair comparison across models, we restrict evaluation to the intersection of samples where all models within a size group achieve non-zero mean AUC. 
For inference, we set the temperature to 0.
For more details, refer to Appendix \ref{appendix_exp_setup}.}


\section{Results and Discussion}
\label{Results and Discussion}

\subsection{Main Results}

\Cref{tab:main_results} summarizes the overall performance and consistency changes across two benchmarks and two base models. In most cases, we observe that RSD achieves a substantial reduction in CV. Smaller models show the largest reduction in CV compared to larger models. On $\tau^2$-bench with Qwen 2.5 3B, RSD reduces AUC CV by 30.9\% and progress CV by 26.1\% compared to the base model. On ToolSandbox 3B, we observe a 50.5\% reduction in AUC CV and a 49.3\% reduction in progress CV. These results suggest that smaller models exhibit higher intrinsic stochasticity in their decision-making processes. In contrast, 7B models exhibit lower baseline variance (e.g., 0.38 CV for ToolSandbox 7B compared to 0.90 for 3B), indicating more stable internal representations. For these larger models, the absolute gains from RSD are smaller because there is less variance to eliminate. 
\camerareview{On $\tau^2$-bench with Qwen-2.5-7B, however, RSD increases both AUC CV and progress CV. A closer analysis of the individual test samples reveals one case where RSD fine-tuning leads to more frequent failures. For this sample, the mean progress decreases from 0.375 to 0.125, while the standard deviation also decreases from 0.518 to 0.354. However, the larger reduction in mean progress causes the CV to increase, despite the lower standard deviation. This illustrates that RSD may reinforce unsuccessful behavior, highlighting that consistency does not necessarily imply correctness.}


We observe that token entropy and semantic entropy produce inconsistent results, with token entropy increasing CV in multiple cases. Minimizing token entropy encourages the model to make confident predictions at each decoding step. However, this token-level confidence does not translate to trajectory-level consistency. This occurs because token entropy measures local uncertainty at the vocabulary level without considering semantic meaning. Semantic entropy partially addresses this issue by clustering similar responses using NLI, but using this signal as a reward still produces inconsistent behavior on $\tau^2$-bench ($-15.9\%$ on 3B and $+0.5\%$ on 7B) and achieves only modest reductions on ToolSandbox ($-7.0\%$ on 3B and $-6.0\%$ on 7B). The fundamental limitation of both methods is that they operate myopically at the level of individual turns. Agents can therefore be locally consistent while remaining globally inconsistent. RSD directly addresses this issue through its trajectory-level formulation (\Cref{def:rsd}) and empirically improves consistency across benchmarks.

We compare RSD with SFT to investigate whether access to ground-truth labels and expert demonstrations leads to more consistent behavior. We observe that on $\tau^2$-bench, SFT reduces CV ($-28.0\%$ for 3B and $-6.5\%$ for 7B), but on ToolSandbox it increases CV ($+17.9\%$ for 7B).

Our training procedure estimates RSD from rollouts generated by a strong base agent that already makes non-trivial progress on the target benchmarks. As a result, the RSD network is trained only on trajectories sampled from a competent policy, and GRPO updates move the policy locally within this region of the behavior space. In particular, degenerate policies that always output a fixed response and achieve near-zero progress do not appear in the rollout distribution and are therefore not explicitly assigned low RSD. This helps explain why optimizing RSD in our experiments does not collapse to trivial “always fail consistently” solutions, even though such policies would theoretically achieve low RSD. Exploring from a randomly initialized or extremely weak policy could expose different failure modes, and avoiding trivial but consistent solutions in that regime would likely require combining RSD with correctness-aligned rewards. \camerareview{For more information on CV calculation and sample count, refer to \Cref{appendix:eval_metrics}.}

\subsection{Hybrid Training Ablation}
\toreview{
\Cref{tab:hybrid_ablation} investigates whether combining RSD with task-specific progress rewards can improve upon pure RSD. We combine the reward from RSD (weight $\lambda$) and progress delta at each step (weight $1-\lambda$) on ToolSandbox for Qwen-2.5-7b model. We observe a counterintuitive result where RSD ($\lambda = 1$) achieves both the lowest CV and the best performance. As we decrease $\lambda$ to incorporate progress rewards, both consistency and performance degrade monotonically. We suspect this is happening because of progress delta being a very sparse reward. Agent sometimes makes the most amount of progress in the final turns which causes previous turns to receive near 0 rewards, causing the signal to be noisy.
}

\begin{table}[H]
\caption{Hybrid ablation on ToolSandbox Qwen-2.5-7B. The RSD weight $\lambda$ controls the balance between RSD consistency reward ($\lambda$) and progress delta reward ($1-\lambda$). $\uparrow \downarrow$ shows relative change from pure RSD ($\lambda=1.0$). The red upwards arrow \textcolor{red}{$\uparrow$} indicates increased in variation.}
\label{tab:hybrid_ablation}
\centering
\setlength{\tabcolsep}{2.5pt}
\small
\begin{tabular}{lccccccc}
\toprule
\textbf{Model} & \textbf{AUC} & \textbf{AUC CV} & \textbf{Progress} & \textbf{Prog CV} & \textbf{$pass$@$1$} & \textbf{$pass$\textasciicircum{}$3$} \\         
\midrule
RSD ($\lambda$=1.0) & \textbf{0.555} & \textbf{0.28} & 0.562 & \textbf{0.27} & 25.0\% & 28.6\%\\

Hybrid ($\lambda$=0.8) & 0.533 & 0.36{\scriptsize\textcolor{red}{$\uparrow$29.3\%}} & 0.541 & 0.33{\scriptsize\textcolor{red}{$\uparrow$21.1\%}} & 26.6\% & 28.6\% \\

Hybrid ($\lambda$=0.5) & 0.532 & 0.37{\scriptsize\textcolor{red}{$\uparrow$33.6\%}} & 0.539 & 0.35{\scriptsize\textcolor{red}{$\uparrow$28.2\%}} & 29.7\% & 21.4\% \\

Hybrid ($\lambda$=0.2) & 0.518 & 0.42{\scriptsize\textcolor{red}{$\uparrow$51.0\%}} & 0.536 & 0.35{\scriptsize\textcolor{red}{$\uparrow$27.0\%}} & 23.4\% & 21.4\% \\

\bottomrule
\end{tabular}
\end{table}

\subsection{Evaluating finetuned models using RSD}
\toreview{
To validate that finetuning with RSD as a reward signal reduces the actual inter-run semantic inconsistency captured by our metric, we evaluate the RSD score directly on the outputs of each finetuned model. For each test sample, the agent is run 8 times independently, and pairwise NLI (gpt-4.1) is used to classify whether outputs at each turn are functionally equivalent (entailment), contradictory, or neutral.

\Cref{tab:rsd_eval} reports the overall RSD scores across all methods. On ToolSandbox, RSD finetuning reduces the inter-run inconsistency by 52.8\% (from 0.672 to 0.317), the largest reduction among all methods. 
Token entropy and semantic entropy also reduce inconsistency ($-40.5\%$ and $-34.4\%$ respectively), confirming that all uncertainty-based rewards can improve behavioral consistency, though RSD is most effective because it directly optimizes the trajectory-level consistency objective. Notably, SFT shows negligible improvement ($-1.1\%$) on ToolSandbox despite memorizing successful trajectories, suggesting that supervised learning alone does not enforce output consistency.

On $\tau^2$-bench, the baseline RSD is substantially higher ($0.873$), reflecting the compounding nature of multi-turn airline service tasks where small early differences cascade into divergent trajectories. All methods show limited improvement, with SFT achieving the largest reduction ($-8.3\%$) by anchoring the model to fixed expert demonstrations. 

\begin{table}[H]
\caption{Direct RSD evaluation (lower = more consistent) across 8 independent trials per sample, measured via gpt-4.1.}
\label{tab:rsd_eval}
\centering
\small
\begin{tabular}{lcccc}
\toprule
& \multicolumn{2}{c}{\textit{ToolSandbox 7B}} & \multicolumn{2}{c}{\textit{$\tau^2$-bench 7B}} \\
\cmidrule(lr){2-3} \cmidrule(lr){4-5}
\textbf{Model} & \textbf{RSD} & \textbf{Reduction} & \textbf{RSD} & \textbf{Reduction} \\
\midrule
Base Qwen & 0.672 & — & 0.873 & — \\
Token Entropy & 0.400 & $-$40.5\% & 0.877 & +0.5\% \\
Semantic Entropy & 0.440 & $-$34.4\% & 0.860 & $-$1.5\% \\
SFT & 0.664 & $-$1.1\% & \textbf{0.800} & $-$8.3\% \\
Hybrid ($\lambda$=0.8) & 0.387 & $-$42.5\% & 0.867 & $-$0.7\% \\
RSD (ours) & \textbf{0.317} & $-$52.8\% & 0.875 & +0.2\% \\
\bottomrule
\end{tabular}
\end{table}

}

\section{Conclusion and Limitation}
\label{Conclusion and Limitation}
\toreview{We introduced Recursive Semantic Divergence (RSD), an unsupervised trajectory-level consistency metric for LLM agents that captures how behavioral inconsistencies propagate across future interaction turns. Unlike existing correctness-based or token-level approaches, RSD provides a recursive estimate of semantic divergence between independent executions from the same interaction-history state, and can additionally be used as a training signal to reinforce consistent agent behavior. Experiments across multiple benchmarks demonstrate that optimizing RSD reduces trajectory variability.

Nevertheless, our approach has several limitations. RSD may require retraining or adaptation when transferred to new domains or environments. Moreover, consistency does not necessarily imply correctness, and therefore the metric cannot distinguish consistently failing behaviors from consistently successful ones. Finally, the reliability of RSD depends heavily on the accuracy and calibration of the underlying NLI model used to estimate semantic disagreement.}

We use LLMs to polish the paper, write portions of the training and inference code, and verify the proofs presented in the paper.


\begin{ack}
This work received no external funding and the authors declare no
competing interests beyond their disclosed affiliations.
\end{ack}

\bibliographystyle{plainnat}
\bibliography{references}

\newpage
\appendix

\section*{Appendix Contents}
\startcontents[appendix]
\begin{spacing}{1.3}
\printcontents[appendix]{}{2}{}
\end{spacing}
\newpage

\section{Appendix}

\subsection{Algorithm for RSD training and inference}
\label{Algorithm for RSD training and inference}
The following pseudocode details the algorithm for RSD training.

\begin{algorithm}[H]
\caption{RSD Training: Learning the Consistency Metric}
\label{alg:rsd_training}
\begin{algorithmic}[1]
\REQUIRE Target agent policy $\pi$, $M$ rollouts per task, NLI classifier, \toreview{pretrained} text encoder $\text{Enc}$
\REQUIRE Hyperparameters: $\gamma$, $\alpha$, $\beta$, $n$ (TD steps), learning rate $\eta$
\STATE \textbf{Data collection:} For each task, collect $M$ independent rollouts $\{\tau^{(1)}, \ldots, \tau^{(M)}\}$ from $\pi$
\STATE \textbf{NLI labeling:} For all trajectory pairs $(i,j)$ within the task, compute $d_{\text{NLI}}(a_t^{(i)}, a_t^{(j)})$ at each turn $t$
\STATE Initialize MICoNetwork parameters $\theta$, target network $\bar{\theta} \leftarrow \theta$
\FOR{training step $= 1, \ldots, N$}
    \STATE Sample a task and two trajectories $\tau^{(i)}, \tau^{(j)}$ for the same task
    \STATE Sample a turn $t$ within the trajectory
    \STATE \toreview{Compute embeddings: $z^{s^{(i)}_{t}} \leftarrow \text{Enc}(s_t^{(i)})$, \ $z^{s^{(j)}_{t}}\leftarrow \text{Enc}(s_t^{(j)})$}
    \STATE Compute $n$-step TD target: $\toreview{y_t} \leftarrow \sum_{k=0}^{n-1} \gamma^k d_{\text{NLI}}(a_{t+k}^{(i)}, a_{t+k}^{(j)}) + \gamma^n \hat{U}_{\bar{\theta}}(s_{t+n}^{(i)}, s_{t+n}^{(j)})$
    \STATE Update $\theta$ to minimize $(\hat{U}_\theta(s_t^{(i)}, s_t^{(j)}) - \toreview{y_t})^2$
    \STATE Soft update the target network
\ENDFOR
\RETURN Trained MICoNetwork $\phi_\theta$
\end{algorithmic}
\end{algorithm}

The following pseudocode details the algorithm for RSD inference.
\begin{algorithm}[H]
\caption{RSD Inference: Per-State Consistency Scoring}
\label{alg:rsd_inference}
\begin{algorithmic}[1]
\REQUIRE Trained MICoNetwork $\phi_\theta$, \toreview{pretrained} text encoder $\text{Enc}$, norm coefficient $\alpha$
\REQUIRE \toreview{Interaction history state} \toreview{$s$} 
\STATE \toreview{$z^{s} \leftarrow \text{Enc}(s)$}
\STATE \toreview{$\phi_s \leftarrow \phi_\theta(z^s)$}
\STATE $\hat{U}(s,s) \leftarrow 2\alpha \|\phi_s\|^2$
\RETURN $\hat{U}(s,s)$
\end{algorithmic}
\end{algorithm}

\subsection{Theoretical Analysis}
\label{appendix:theory}

We establish that the RSD fixed-point equation converges to a well-defined quantity with the NLI changes, and the quantity provides meaningful guarantees about agent behavior. Throughout, we let $\tau^{(i)}, \tau^{(j)} \sim \Pi(\cdot | s)$ denote two independent trajectory rollouts from state $s$, with $a_t^{(i)}, a_t^{(j)}$ the actions taken at step $t$ in each rollout.

\subsubsection{Convergence}
We first verify that the recursive equation (\Cref{eq:rsd_general}) has a unique solution. Recall the RSD fixed-point equation:
$$U(s_t^{(i)}, s_t^{(j)}) = \E_{a_t^{(i)} \sim \pi(\cdot|s_t^{(i)}),\; a_t^{(j)} \sim \pi(\cdot|s_t^{(j)})}\big[d_{\text{NLI}}(a_t^{(i)}, a_t^{(j)})\big] + \gamma \cdot \E\big[U(s_{t+1}^{(i)}, s_{t+1}^{(j)})\big].$$
Define the Bellman operator $\mathcal{T}$ acting on bounded functions $U: \mathcal{H} \times \mathcal{H} \to \R_{\geq 0}$ by:
$$(\mathcal{T}U)(s_t^{(i)}, s_t^{(j)}) = \E_{a_t^{(i)} \sim \pi(\cdot|s_t^{(i)}),\, a_t^{(j)} \sim \pi(\cdot|s_t^{(j)})}\big[d_{\text{NLI}}(a_t^{(i)}, a_t^{(j)})\big] + \gamma \cdot \E\big[U(s_{t+1}^{(i)}, s_{t+1}^{(j)})\big],$$
where $s_{t+1}^{(i)} \sim P(\cdot | s^{(i)}, a^{(i)})$ and $s_{t+1}^{(j)} \sim P(\cdot | s^{(j)}, a^{(j)})$ are the respective successor states. Since $d_{\text{NLI}} \in [0,1]$ is bounded and $\gamma < 1$, the operator $\mathcal{T}$ is a $\gamma$-contraction in supremum norm. 

\subsubsection{Disagreement Probability}

Given that $U^*$ is well-defined, we derive what the self-distance $U(s,s)$ implies about agent behaviour. We make use of the \Cref{def:rsd}, where self distance equals to the expected total discounted disagreement between two independent rollouts.

\begin{equation}
U(s,s) = \E_{\tau^{(i)}, \tau^{(j)} \sim \Pi(\cdot|s)}\left[\sum_{t=0}^{\infty} \gamma^t\, d_{\text{NLI}}(a_t^{(i)}, a_t^{(j)})\right] = \E[D(\tau^{(i)}, \tau^{(j)})],
\label{eq:u_equals_expectation}
\end{equation}
where $D(\tau^{(i)}, \tau^{(j)}) \coloneqq \sum_{t=0}^{\infty} \gamma^t\, d_{\text{NLI}}(a_t^{(i)}, a_t^{(j)}) \geq 0$ is the total discounted disagreement.

\setcounter{theorem}{0} 
\begin{theorem}[Trajectory Agreement]\label{thm:agreement_appendix}
Let $D(\tau^{(i)}, \tau^{(j)}) = \sum_{t=0}^{\infty} \gamma^t d_{\text{NLI}}(a_t^{(i)}, a_t^{(j)})$ be the total discounted disagreement between two independent rollouts from state $s$. For any threshold $\delta > 0$:
\begin{equation}
    \Pr\big(D(\tau^{(i)}, \tau^{(j)}) \geq \delta\big) \leq \frac{U(s,s)}{\delta}.
\end{equation}
\end{theorem}

\begin{proof}
Since $D(\tau^{(i)}, \tau^{(j)}) \geq 0$ and $\E[D(\tau^{(i)}, \tau^{(j)})] = U(s,s)$ by \Cref{eq:u_equals_expectation}, Markov's inequality directly gives:
$$\Pr\big(D(\tau^{(i)}, \tau^{(j)}) \geq \delta\big) \leq \frac{\E[D(\tau^{(i)}, \tau^{(j)})]}{\delta} = \frac{U(s,s)}{\delta}.$$
\end{proof}

A useful instantiation is $\delta = \sqrt{U(s,s)}$, which gives $\Pr(D \geq \sqrt{U(s,s)}) \leq \sqrt{U(s,s)}$: both the disagreement threshold and the probability of exceeding it shrink as $U(s,s) \to 0$.

\begin{corollary}[Immediate Contradiction Bound]\label{cor:contradiction}
The probability that two independent actions from state $s$ contradict each other satisfies:
$$\Pr\big(\text{contradiction at step } 0\big) \leq U(s,s).$$
\end{corollary}

\begin{proof}
We first bounding the expected immediate disagreement by $U(s,s)$, then bounding the contradiction probability by the expected immediate disagreement.

Define $D_0(s) \coloneqq \E_{a_0^{(i)}, a_0^{(j)} \sim \pi(\cdot|s)}[d_{\text{NLI}}(a_0^{(i)}, a_0^{(j)})]$ as the expected immediate disagreement. Expanding $D(\tau^{(i)}, \tau^{(j)})$ into its first term and the remainder:
$$D(\tau^{(i)}, \tau^{(j)}) = d_{\text{NLI}}(a_0^{(i)}, a_0^{(j)}) + \gamma \sum_{t=1}^{\infty} \gamma^{t-1}\, d_{\text{NLI}}(a_t^{(i)}, a_t^{(j)}).$$
Since $d_{\text{NLI}} \geq 0$, every term in the remainder is non-negative. Taking expectations of both sides:
$$U(s,s) = \E[D(\tau^{(i)}, \tau^{(j)})] = D_0(s) + \underbrace{\gamma \E\left[\sum_{t=1}^{\infty} \gamma^{t-1}\, d_{\text{NLI}}(a_t^{(i)}, a_t^{(j)})\right]}_{\geq\, 0}.$$
Therefore $D_0(s) \leq U(s,s)$.

Then, let $p_{\text{entail}}$, $p_{\text{neutral}}$, and $p_{\text{contra}}$ denote the probabilities that two independent samples $a_0^{(i)}, a_0^{(j)} \sim \pi(\cdot|s)$ yield bidirectional entailment, neutrality, or contradiction respectively, with $p_{\text{entail}} + p_{\text{neutral}} + p_{\text{contra}} = 1$. By \Cref{def:nli}, $d_{\text{NLI}}$ takes values $0$, $0.5$, and $1$ for these three cases.
$$D_0(s) = 0 \cdot p_{\text{entail}} + 0.5 \cdot p_{\text{neutral}} + 1 \cdot p_{\text{contra}} = 0.5\, p_{\text{neutral}} + p_{\text{contra}}.$$
Since $0.5\, p_{\text{neutral}} \geq 0$, we have $D_0(s) \geq p_{\text{contra}}$.

This results in: 
$$\Pr(\text{contradiction at step } 0) = p_{\text{contra}} \leq D_0(s) \leq U(s,s).$$
\end{proof}

\begin{theorem}[Dominant Action Class]\label{thm:dominant}
If $U(s,s) < 1/4$, there exists a semantic equivalence class $C^*$ such that:
\begin{equation}
    \Pr(\pi(s) \in C^*) \geq 1 - 2U(s,s).
\end{equation}
Here $C^* = C_{k^*}$ where $k^* = \arg\max_k \Pr(\pi(s) \in C_k)$ is the most probable equivalence class under the policy.
\end{theorem}

\begin{proof}
Partition the action space into semantic equivalence classes $\{C_k\}_{k=1}^K$ defined by bidirectional NLI entailment: two actions $a^{(i)}, a^{(j)}$ belong to the same class $C_k$ if and only if $d_{\text{NLI}}(a^{(i)}, a^{(j)}) = 0$ (i.e., they mutually entail each other). Let $p_k = \Pr(\pi(s) \in C_k)$ with $\sum_k p_k = 1$.

We first lower bound on collision probability. For two independent samples $a_0^{(i)}, a_0^{(j)} \sim \pi(\cdot|s)$, the probability they fall in the same class is:
$$\Pr(a_0^{(i)} \sim a_0^{(j)}) = \sum_{k=1}^K p_k^2.$$
When the samples fall in different classes, they are either neutral or contradictory. By \Cref{def:nli}, $d_{\text{NLI}}(a_0^{(i)}, a_0^{(j)}) \geq 0.5$ in both cases. Therefore, we can lower-bound the expected immediate disagreement:


\begin{equation}
\label{eqn:lower_bound_d0}
\begin{aligned}
D_0(s)
&= \E[d_{\text{NLI}}(a_0^{(i)}, a_0^{(j)})] \\
&= \underbrace{\E[d_{\text{NLI}} \mid a_0^{(i)} \sim a_0^{(j)}]}_{=\,0}
   \Pr(a_0^{(i)} \sim a_0^{(j)}) \\
&\quad + \E[d_{\text{NLI}} \mid a_0^{(i)} \not\sim a_0^{(j)}]
   \Pr(a_0^{(i)} \not\sim a_0^{(j)}) \\
&\geq 0.5\,\Pr(a_0^{(i)} \not\sim a_0^{(j)}) \\
&= 0.5\left(1-\sum_{k=1}^{K}p_k^2\right).
\end{aligned}
\end{equation}

By \Cref{cor:contradiction}, $D_0(s) \leq U(s,s)$. Substituting into the inequality from \Cref{eqn:lower_bound_d0}:
$$0.5\left(1 - \sum_{k=1}^K p_k^2\right) \leq U(s,s).$$
Rearranging:
\begin{equation}
\label{eqn:step2_dominant_class}
    \sum_{k=1}^K p_k^2 \geq 1 - 2U(s,s).
\end{equation}

Now, we show the probability of the dominant class. For any probability distribution $(p_1, \ldots, p_K)$ with $\sum_k p_k = 1$, the following holds:
$$\max_k p_k \geq \sum_{k=1}^K p_k^2.$$

Combining with \Cref{eqn:step2_dominant_class}:
$$\max_k p_k \geq \sum_k p_k^2 \geq 1 - 2U(s,s).$$

The class $C^* = C_{k^*}$ achieving this maximum satisfies $\Pr(\pi(s) \in C^*) \geq 1 - 2U(s,s)$. The condition $U(s,s) < 1/4$ ensures this bound exceeds $1/2$, meaning $C^*$ contains a strict majority of the policy's probability mass.
\end{proof}

\subsubsection{Approximation Error}

The guarantees above hold for the true fixed point $U^*$. In practice, the learned approximation $\hat{U}_\theta$ differs from $U^*$ due to two sources of error: (1) the state embedding $\phi$ loses information relative to the full interaction history, and (2) the NLI classifier $\hat{d}_{\text{NLI}}$ is imperfect. We quantify how these per-step errors propagate through the recursive structure.

\camerareview{Here, $d^*_{\text{NLI}}$ denotes the ideal semantic disagreement function, while $\hat{d}_{\text{NLI}}$ denotes its classifier-based approximation.}

Let $\epsilon_\phi$ denote the worst-case per-step error from embedding compression:
$$
\epsilon_\phi = \sup_{s^{(i)}, s^{(j)}} \left| \E[d^*_{\text{NLI}}(a^{(i)}, a^{(j)}) \mid s^{(i)}, s^{(j)}] - \E[d^*_{\text{NLI}}(a^{(i)}, a^{(j)}) \mid \phi(s^{(i)}), \phi(s^{(j)})] \right|
$$

and let $\epsilon_{\text{NLI}}$ denote the worst-case NLI classification error:
$$
\epsilon_{\text{NLI}} = \sup_{a^{(i)}, a^{(j)}} \left| \hat{d}_{\text{NLI}}(a^{(i)}, a^{(j)}) - d^*_{\text{NLI}}(a^{(i)}, a^{(j)}) \right|
$$

\camerareview{For the following analysis, we assume that the state embedding preserves the policy-induced distribution of successor embeddings, such that states with identical embeddings have the same distribution over future embedded states. We further assume that the learned network $\hat{U}_\theta$ exactly represents the fixed point of the approximate Bellman operator. These assumptions allow us to isolate the effects of embedding compression and NLI classification errors.}

\begin{theorem}[Error Propagation]
\label{thm:error}
\camerareview{Under the assumptions above,} the approximation error satisfies:
\begin{equation}
    \left|\hat{U}_\theta(s, s) - U^*(s, s)\right| \leq \frac{\epsilon_\phi + \epsilon_{\text{NLI}}}{1 - \gamma}.
\end{equation}
\end{theorem}

\begin{proof}
Let $\tilde{\mathcal{T}}$ denote the approximate Bellman operator that uses $\hat{d}_{\text{NLI}}$ in place of $d^*_{\text{NLI}}$ and conditions on embeddings $\phi(s)$ rather than full states.

\camerareview{We define both operators on full state pairs using the same policy-induced transition kernel. Under the embedding assumption, the fixed point of $\tilde{\mathcal{T}}$ can be represented using only the embedded states.}

Since each modification introduces at most $\epsilon_{\text{NLI}}$ and $\epsilon_\phi$ error respectively in the per-step cost, while the discounted recursive term is shared, we have for any function $U$ and any state pair $(s^{(i)}, s^{(j)})$:
$$
|(\mathcal{T}U)(s^{(i)}, s^{(j)}) - (\tilde{\mathcal{T}}U)(s^{(i)}, s^{(j)})| \leq \epsilon_\phi + \epsilon_{\text{NLI}}.
$$

Both operators are $\gamma$-contractions (since only the per-step cost differs, not the discount structure). Let $U^*$ and $\tilde{U}^*$ be their respective fixed points. Evaluating at $U^*$:
$$
\tilde{\mathcal{T}}U^* = \tilde{\mathcal{T}}U^* - \mathcal{T}U^* + U^*,
$$
so $\|\tilde{\mathcal{T}}U^* - U^*\|_\infty \leq \epsilon_\phi + \epsilon_{\text{NLI}}$. Applying $\tilde{\mathcal{T}}$ repeatedly and using contractivity:
$$
\|\tilde{U}^* - U^*\|_\infty =
\lim_{n\to\infty}\|\tilde{\mathcal{T}}^n U^* - U^*\|_\infty
\leq \sum_{k=0}^{\infty} \gamma^k (\epsilon_\phi + \epsilon_{\text{NLI}})
= \frac{\epsilon_\phi + \epsilon_{\text{NLI}}}{1-\gamma}.
$$

\end{proof}

For $\gamma = 0.8$, the amplification factor is $1/(1-\gamma) = 5$, \camerareview{showing how per-step approximation errors accumulate through the recursive structure.}

\camerareview{This bound holds under exact Bellman approximation. In practice, the learned network may introduce additional fitting error, which is not captured by the bound.}

\subsection{Experimental Setup}
\label{appendix_exp_setup}

We evaluate RSD on $\tau^2$-bench \citep{barres2025tau}, and ToolSandbox \citep{lu2024toolsandbox} to measure whether the learned metric captures trajectory-level consistency and if it can serve as a reward signal for improving consistency during fine-tuning. $\tau^2$-bench contains scenarios related to airline customer service agent, while ToolSandbox comprises general-purpose tasks spanning diverse domains, including calendar management and web search. For each benchmark, we run 8 independent trials across 8 test samples (64 trajectories) in total for Qwen-2.5 models (3B and 7B parameters) \citep{qwen2.5}, where each trajectory consists of interleaved user messages, agent actions (including tool calls), and tool outputs.

\subsubsection{RSD Training}

\toreview{
We construct the training set by generating multiple rollouts using temperature set to 1. For $\tau^2$-bench we collect $M=30$ independent rollouts from $13$ task instances, and for ToolSandbox we collect $M=8$ independent rollouts from $29$ task instances. From the resulting trajectories, we extract state-action transition as described in \Cref{sec:pomdp}. Each state $s_t$ consists of the full interaction history up to timestep $t$, i.e,  $(o_0, a_0, o_1, a_1, \ldots, o_{t-1}, a_{t-1}, o_t)$, where observations include user messages and tool outputs, and actions include agent responses and tool calls.

We adopt the MICo parameterization from \Cref{eq:mico_dist}, where $z^s$ is computed using a frozen Cohere embed-v4-0 and mapped through a 3-layer MLP into a 256-dimensional latent space. The network is trained to satisfy \Cref{eq:rsd_general} using $n=3$ TD steps and $\gamma=0.8$. To compute pairwise NLI relations between actions, we use gpt-4.1 to label each pair as entailment, contradiction or neutral, which are then mapped to scores as defined in \Cref{def:nli}. The trained models are then used as reward signals to finetune different models. Refer to \Cref{alg:rsd_training} for more detail on the training algorithm. 
}

\subsubsection{Trajectory Divergence and Terminal States}
\label{appendix:trajectory_termination}

\camerareview{As described in \Cref{sec:approximation}, two independently sampled trajectories from the same initial state are paired at the same transition depth. The shared index $t$ does not indicate semantically corresponding decisions across the two rollouts. After divergence, each trajectory independently follows its own successor states, without attempting to realign the trajectories. Consider two independent rollouts from the same initial state:}

{
\begin{align*}
\tau^{(i)} &: \texttt{hotel\_booking}
\rightarrow \texttt{flight\_booking}
\rightarrow \texttt{confirmed},\\
\tau^{(j)} &: \texttt{flight\_booking}
\rightarrow \texttt{hotel\_booking}
\rightarrow \texttt{confirmed}.
\end{align*}
}

\camerareview{At the first two transition depths, the paired actions are classified as contradictions, while the final responses are classified as entailment. The TD target therefore accumulates the discounted semantic disagreement between the different execution paths rather than attempting to match semantically equivalent actions at different positions.}

\camerareview{For variable-length trajectories, if either trajectory terminates before the $n$-step TD window, the target accumulates disagreement only over the valid paired transitions. Specifically, if only $m<n$ paired transitions are available before termination, the TD target in \Cref{eq:td_y} becomes}

{
\begin{equation}
    y_t = \sum_{k=0}^{m-1} \gamma^k
    d_{\text{NLI}}(a_{t+k}^{(i)}, a_{t+k}^{(j)}).
    \label{eq:terminal_td}
\end{equation}
}

\camerareview{The bootstrap term is omitted because no paired successor state exists beyond the terminal transition. No padding or additional actions beyond termination are introduced.}

\subsubsection{Baselines}
\label{appendix:baselines}

\toreview{
We compare RSD against three classes of methods:

\textbf{Token-level entropy} \citep{malinin2020uncertainty}: We compute the average token-level entropy across all tokens in the agent's response at each turn. This provides a measure of uncertainty at the token generation level but does not capture semantic consistency across multiple turns.

\textbf{Semantic entropy} \citep{kuhn2023semantic}: We cluster agent responses using bidirectional NLI entailment to identify semantic equivalence classes, then compute the entropy over these clusters. Unlike token entropy, this captures semantic variation, but only at single-response granularity without considering future trajectory divergence.

\textbf{Supervised fine-tuning (SFT)}: We fine-tune the base Qwen-2.5 models on expert trajectories from both benchmarks using next-token prediction. This provides a baseline that improves task success through supervised learning but does not explicitly optimize for consistency.

For token entropy, semantic entropy, and RSD, we fine-tune the base models using Group Relative Policy Optimization (GRPO) \citep{Shao2024DeepSeekMath}, using each metric as a reward signal. We minimize token and semantic entropy via negative entropy rewards, and minimize RSD self-distance $-\hat{U}_\theta(s,s)$. All methods use the same base models (Qwen-2.5-3B and Qwen-2.5-7B), share GRPO hyperparameters, and are evaluated under identical trajectory collection settings with 8 trials per task for a fair comparison.
}

\subsubsection{Evaluation Metrics}
\label{appendix:eval_metrics}

\toreview{
We adopt the evaluation framework from \citet{chong2026talk} to evaluate the agent behaviour, as it captures both intermediate progress and final outcomes. We report two task performance metrics, area under the progress curve (AUC) which measures how quickly the agent makes progress, and the final progress rate. 

To measure behavioral consistency, we compute the coefficient of variation (CV = $\sigma / \mu$) of performance metrics across 8 independent trials per task, where a lower CV indicates more consistent behavior. To ensure a fair comparison across models, we restrict evaluation to the intersection of samples where all models within a size group achieve non-zero mean AUC. This avoids undefined CV where $\mu = 0$, and also prevents models outputting trivial solutions from appearing more consistent. This guarantees that every model is evaluated on identical samples where it demonstrates at least partial task engagement. Additionally, we report pass@k and pass\textasciicircum{}$k$  metrics \citep{yao2024tau,barres2025tau}. For inference, we set the temperature to 0. 
\camerareview{For $pass^3$, we report scores over the intersection of samples where all models within each benchmark--model-size group achieve non-zero mean task progress across trials. This ensures that all methods are evaluated on the same subset of samples.}
}

\subsection{Evaluation on Qwen-3.5 Models}
\label{appendix:qwen35}

We evaluate RSD on Qwen-3.5 models (4B and 9B parameters) across both $\tau^2$-bench and ToolSandbox to investigate whether the consistency improvements extend across model families and scales. We compare RSD fine-tuning against the corresponding base models using AUC, final progress, and their respective coefficients of variation, as described in \Cref{appendix:eval_metrics}.

\begin{table}[H]
\caption{Performance and consistency, measured by the coefficient of variation (CV), for Qwen-3.5 models across both benchmarks. The blue downwards arrow \textcolor{blue}{$\downarrow$} indicates a reduction in variation, while the red upwards arrow \textcolor{red}{$\uparrow$} indicates increased variation relative to the baseline.}
\label{tab:qwen35_results}
\centering
\small
\begin{tabular}{llcccc}
\toprule
& \textbf{Model} & \textbf{AUC} & \textbf{AUC CV} & \textbf{Progress} & \textbf{Progress CV} \\
\midrule
\multicolumn{6}{c}{\textit{$\tau^2$-bench Qwen-3.5-4B}} \\
\midrule
& Base Qwen (baseline) & 0.627 & 0.444 & 0.705 & 0.454 \\
& RSD & \textbf{0.686} & \textbf{0.355}{\scriptsize\textcolor{blue}{$\downarrow$20.0\%}} & \textbf{0.768} & \textbf{0.359}{\scriptsize\textcolor{blue}{$\downarrow$20.9\%}} \\
\midrule
\multicolumn{6}{c}{\textit{$\tau^2$-bench Qwen-3.5-9B}} \\
\midrule
& Base Qwen (baseline) & 0.607 & 0.480 & 0.686 & 0.483 \\
& RSD & \textbf{0.650} & \textbf{0.428}{\scriptsize\textcolor{blue}{$\downarrow$10.8\%}} & \textbf{0.720} & \textbf{0.427}{\scriptsize\textcolor{blue}{$\downarrow$11.6\%}} \\
\midrule
\multicolumn{6}{c}{\textit{ToolSandbox Qwen-3.5-4B}} \\
\midrule
& Base Qwen (baseline) & 0.405 & 0.742 & 0.412 & 0.738 \\
& RSD & \textbf{0.447} & \textbf{0.662}{\scriptsize\textcolor{blue}{$\downarrow$10.8\%}} & \textbf{0.451} & \textbf{0.662}{\scriptsize\textcolor{blue}{$\downarrow$10.3\%}} \\
\midrule
\multicolumn{6}{c}{\textit{ToolSandbox Qwen-3.5-9B}} \\
\midrule
& Base Qwen (baseline) & 0.488 & 0.758 & 0.503 & \textbf{0.720} \\
& RSD & \textbf{0.507} & \textbf{0.738}{\scriptsize\textcolor{blue}{$\downarrow$2.6\%}} & \textbf{0.515} & 0.721{\scriptsize\textcolor{red}{$\uparrow$0.1\%}} \\
\bottomrule
\end{tabular}
\end{table}

\Cref{tab:qwen35_results} summarizes the performance and consistency changes for Qwen-3.5 models. We observe that RSD reduces AUC CV across both benchmarks and model sizes, while simultaneously improving AUC and final progress. On $\tau^2$-bench, RSD reduces AUC CV by 20.0\% and progress CV by 20.9\% for Qwen-3.5-4B, while Qwen-3.5-9B achieves reductions of 10.8\% and 11.6\%, respectively. On ToolSandbox, Qwen-3.5-4B achieves a 10.8\% reduction in AUC CV and a 10.3\% reduction in progress CV. For Qwen-3.5-9B, the improvements are smaller, with a 2.6\% reduction in AUC CV and a marginal 0.1\% increase in progress CV.

These results follow a similar trend to those observed in \Cref{tab:main_results}, where smaller models generally benefit more from RSD fine-tuning. The smaller improvements for Qwen-3.5-9B may be due to the same fine-tuning budget being used across model sizes, which may be insufficient for larger models. Overall, the results suggest that RSD can improve consistency and task performance across different model families, although the magnitude of improvement varies with model size and benchmark.

\subsection{Validation of NLI Distance}
\label{appendix:nli_validation}

\subsubsection{Three-Class NLI Discretization}

\camerareview{We adopt the three-class NLI formulation from \Cref{def:nli} to distinguish functionally equivalent, contradictory, and partially related actions. For tool-calling actions, we consider both the semantics of the tool and its arguments rather than relying on syntactic similarity. For example, two calls to the same tool with equivalent arguments are treated as entailment, while calls with conflicting arguments that lead to different outcomes are treated as contradiction. Neutrality captures cases where actions share similar intent but their functional equivalence is ambiguous.}

\camerareview{To evaluate whether the three-class discretization is sufficient, we sample 12,000 action pairs from the $\tau^2$-bench Airline domain and compare the proposed three-class scoring with a finer-grained 10-class scoring system (0--9). We observe that 42.5\% of the scores fall in the range 0--1, 33.9\% fall in the range 7--9, and only 6.1\% fall in the middle range (4--6). We compute RSD using both scoring systems independently and compare the resulting rankings. The Spearman rank correlation is 0.97, indicating that the three-class formulation largely preserves the relative ordering obtained using finer-grained scoring.}

\camerareview{Furthermore, finer-grained scoring requires the NLI classifier to distinguish between more semantic categories, potentially increasing sensitivity to calibration errors. We also hypothesize that finer-grained rewards may reduce the separation between consistent and inconsistent behaviors in group-relative advantage estimation during GRPO fine-tuning. These considerations motivate our use of the three-class formulation.}

\subsubsection{NLI Validation for Tool-Calling Actions}

\camerareview{We further evaluate the reliability of GPT-4.1 NLI classifications for tool-calling actions. We programmatically categorize action pairs into three groups: calls to the same tool with identical arguments, calls to the same tool with different arguments, and calls to different tools. \Cref{tab:nli_tool_validation} reports the distribution of NLI classifications for each category.}

\begin{table}[H]
\caption{GPT-4.1 NLI classification distributions for tool-calling action pairs.}
\label{tab:nli_tool_validation}
\centering
\small
\begin{tabular}{lcccc}
\toprule
\textbf{Category} & \textbf{N} & \textbf{Entailment} & \textbf{Neutral} & \textbf{Contradiction} \\
\midrule
Same tool + same args & 1,602 & 97.6\% & 0.1\% & 2.3\% \\
Same tool + different args & 343 & 6.1\% & 0.0\% & 93.9\% \\
Different tools & 695 & 0.0\% & 0.0\% & 100.0\% \\
\bottomrule
\end{tabular}
\end{table}

\camerareview{We observe that 97.6\% of action pairs with the same tool and arguments are classified as entailment, while 93.9\% of pairs with the same tool but different arguments are classified as contradiction. All sampled pairs involving different tools are classified as contradiction. These results indicate that GPT-4.1 can distinguish between different tool-calling actions and their arguments in the evaluated samples.}

\camerareview{We additionally manually annotate 30 randomly sampled action pairs across the three NLI categories (10 entailment, 15 contradiction, and 5 neutral) for functional equivalence. We observe 96.7\% agreement between the manual annotations and GPT-4.1 judgments.}

\subsubsection{Cross-Model NLI Agreement}

\camerareview{To investigate whether the NLI classifications depend on the choice of GPT-4.1, we compare its predictions with other LLMs and an open-source NLI model on randomly sampled action pairs from ToolSandbox. \Cref{tab:nli_model_agreement} reports the agreement between GPT-4.1 and the alternative models.}

\begin{table}[H]
\caption{Cross-model NLI agreement on sampled ToolSandbox action pairs.}
\label{tab:nli_model_agreement}
\centering
\small
\begin{tabular}{lc}
\toprule
\textbf{Model Pair} & \textbf{Agreement (\%)} \\
\midrule
GPT-4.1 vs Kimi-K2.5 & 95.0 \\
GPT-4.1 vs GPT-5.1 & 98.0 \\
GPT-4.1 vs GPT-4o & 97.5 \\
GPT-4.1 vs GPT-4o-mini & 97.0 \\
GPT-4.1 vs DeBERTa-v3-large & 41.6 \\
\bottomrule
\end{tabular}
\end{table}

\camerareview{We observe high agreement between GPT-4.1 and the other LLM-based judges, ranging from 95.0\% to 98.0\%. In contrast, DeBERTa-v3-large achieves substantially lower agreement (41.6\%). Unlike general-purpose LLMs, traditional NLI models such as DeBERTa are primarily trained on natural-language inference datasets and may not reliably capture the functional equivalence of structured tool calls. These results suggest that the NLI classifications are not specific to GPT-4.1.}


\end{document}